\documentclass[11pt]{article}

\usepackage[numbers]{natbib}
\usepackage[utf8]{inputenc} 
\usepackage[T1]{fontenc}    
\usepackage[colorlinks=true]{hyperref}       
\usepackage{url}            
\usepackage{booktabs}       
\usepackage{amsfonts}       
\usepackage{nicefrac}       
\usepackage{microtype}      
\usepackage{xcolor}         
\usepackage[margin=1in]{geometry}

\usepackage{amsmath}
\usepackage{amssymb}
\usepackage{amsthm}
\usepackage{thmtools}
\DeclareMathOperator*{\argmin}{arg\,min}

\usepackage{dsfont}

\newtheorem{theorem}{Theorem}
\newtheorem{lemma}{Lemma}
\newtheorem{fact}{Fact}
\newtheorem{corollary}{Corollary}

\usepackage{algorithm,algpseudocode}
\floatstyle{ruled}
\newfloat{algorithm}{tbp}{loa}
\providecommand{\algorithmname}{Algorithm}
\floatname{algorithm}{\protect\algorithmname}
\algnewcommand{\Parameters}[1]{\item[\textbf{Parameters:}]{#1}}
\algnewcommand{\Init}[1]{\item[\textbf{Initialize:}]{#1}}

\title{Non-stochastic Bandits With Evolving Observations}
\author{
Yogev Bar-On%
\thanks{Tel Aviv University; \texttt{baronyogev@gmail.com}.}
\and
Yishay Mansour
\thanks{Tel Aviv University and Google Research; \texttt{mansour.yishay@gmail.com}.}
}
\date{}

\begin{document}
\maketitle

\begin{abstract}
We introduce a novel online learning framework that unifies and generalizes pre-established models, such as delayed and corrupted feedback, to encompass adversarial environments where action feedback evolves over time. In this setting, the observed loss is arbitrary and may not correlate with the true loss incurred, with each round updating previous observations adversarially. We propose regret minimization algorithms for both the full-information and bandit settings, with regret bounds quantified by the average feedback accuracy relative to the true loss. Our algorithms match the known regret bounds across many special cases, while also introducing previously unknown bounds.
\end{abstract}


\section{Introduction}\label{sec:intro}
In many sequential decision problems, the outcomes of actions are not immediately observed, and rather could only be estimated. Those estimations may be changing constantly (usually becoming more accurate) until certainty about the outcome is achieved.

This situation is common in financial settings, such as options trading, where the underlying asset price at the time of exercising can only be estimated at the time of the trade. Another more modern example is trading on blockchain systems (see, e.g., \citep{bar2023uniswap}), where new blocks are created immediately but may be deleted with a low probability after a certain time (due to "forking" \citep{neudecker2019short}).

Another application is online advertising, where the value of an ad may evolve over time. For example, a user's click indicates some positive value. If the user completes an online form at a later time, the value increases. The value can then increase again if the user completes a purchase, or decrease if the user leaves without further action.

While working under these conditions, the theory of online learning with delayed feedback \citep{cesa2016delay,thune2019nonstochastic} is useful for making informed decisions. However, there is still a big gap between theory and practice - primarily since the estimations of action values are not considered before the true value becomes known.

\paragraph{Evolving feedback}
To bridge this gap, we propose a new framework for online decision-making in which the feedback on actions taken by the agent evolves and changes over time. Notably, the feedback can change retroactively, overriding previous observations. This applies to and generalizes various established feedback mechanisms. Those include delayed feedback, where all the information about the loss (the action's result) is revealed at a future time; composite feedback, where the loss is revealed monotonically over time; and corrupted feedback, where the true loss is never revealed.

We investigate online learning environments with $K$ actions over $T$ rounds. An oblivious adversary chooses ahead of time not only the true loss $\ell_t \in [0,1]^K$ for round $t$, but also the \emph{feedback loss} $\ell_\tau^{(t)}\in [0,1]^K$, which represents the new observation regarding the actions taken previously at any step $\tau\leq t$, overriding any previous observations.

As usual, at each round $t$ the agent chooses an action $a_t\in[K]$. The observations do not affect the suffered loss, and our objective is still to minimize the expected regret $R(T)$ in comparison to the best true loss in hindsight:
\[
R(T) \triangleq \max_{a\in[K]}\mathbb{E}\left[{\sum_{t=1}^T{\ell_{t,a_t} - \ell_{t,a}}}\right].
\]

In our model, the observed losses may or may not correlate with the true loss. In the latter case, the agent gains no information and thus no strategy can guarantee a low regret. Fortunately, in real-life situations we expect the observations to be good estimates of the true loss and be more accurate as time progresses. Hence our algorithms' regret bound is smaller the more accurate the feedback is, and depends on the accuracy term 
\[
\Lambda = \sum_{t=1}^{T}{\sum_{\tau=1}^{t}{\min\left\{1, \left\|\ell_\tau - \ell^{(t)}_\tau \right\|_2\right\}}}.
\]

\subsection{Contributions and outline}\label{subsec:contributions}
\paragraph{Full-information setting} We start with presenting an Exponential Weights \citep{cesa1997use,cesa2006prediction} variant for the full-information setting in Section \ref{sec:ewa}, where all the feedback generated by the adversary is revealed to the agent, and show a regret bound better than $\widetilde{O}\left(\sqrt{\Lambda}\right)$, where $\widetilde{O}$ hides logarithmic terms.

\paragraph{Bandit setting} We then present a Follow-The-Regularized-Leader (FTRL) \citep{abernethy2008competing,orabona2019modern} variant for the bandit setting in Section \ref{sec:ftrl}, where only the feedback generated for actions taken by the agent is revealed, and show a regret bound of $\widetilde{O}\left(\sqrt{KT+\Lambda}\right)$. Our novel analysis creatively adapts methods from the delayed setting, where we quantify the information revealed at each step using the feedback accuracy, instead of a binary revealed/not revealed. 

We can consider the delayed setting as a special case, where if the delay of round $\tau$ is $d_\tau$, then $\ell^{(t)}_\tau = \ell_\tau$ for any $t\geq \tau+d_\tau$, and $\ell^{(t)}_\tau = 0$ otherwise. Hence, $\Lambda=\sum_{t=1}^{T} d_t$ captures the total delay. Thus, our regret bounds are optimal (up to logarithmic terms) when used in the delayed setting \citep{cesa2016delay}.

\paragraph{Applications} We show more special cases in Section \ref{sec:bounds}. A key benefit to our framework is that it naturally supports any combination of those applications, for example, delayed feedback that is sometimes corrupted.
\begin{itemize}
\item In the optimistic delayed feedback environment \citep{flaspohler2021online,hsieh2022multi}, where hints on delayed feedback are available to the agent, we match the existing regret bound in the full-information setting and show the \emph{first regret bound} in the bandit setting.
\item In the corrupted feedback environment \citep{resler2019adversarial,hajiesmaili2020adversarial}, where the true losses are never revealed, we get a standard $\widetilde{O}\left(\sqrt{KT} + \mathcal{C}\right)$ bound, where $\mathcal{C}$ is the corruption budget.
\item In the composite delayed feedback environment \citep{cesa2018nonstochastic,wang2021adaptive}, where the feedback is spread over $d$ partial consecutive observations, we apply our framework to get the optimal $\widetilde{O}\left(\sqrt{(K+d)T}\right)$ regret bound.
\end{itemize}

For clarity, we defer detailed proofs to the appendix.

\subsection{Additional related works}
Online learning under adversarial delayed feedback has been studied extensively both under the full-information \citep{weinberger2002delayed,joulani2013online,joulani2016delay} and bandit \citep{bistritz2019online,zimmert2020optimal,ito2020delay,gyorgy2021adapting,jin2022near,van2022nonstochastic,li2023modified} settings. Our analysis in this work is inspired by a recent work \citep{van2023unified} that unifies the analysis of delayed feedback under many regimes such as linear bandits and Markov decision processes.

Many works study stochastic delayed environments as well \citep{agarwal2011distributed,vernade2020linear,pike2018bandits,gael2020stochastic,lancewicki2021stochastic,tang2024stochastic}, and there are optimal algorithms for both cases simultaneously ("best of both worlds") \citep{masoudian2022best,masoudian2023improved}.

Also generalized by our work is a corrupted adversarial feedback environment, previously studied for the stochastic case as well \citep{lykouris2018stochastic,amir2020prediction,ito2021optimal,he2022nearly}.


\section{Evolving Exponential Weights}\label{sec:ewa}
We start by presenting a simple regret minimization algorithm for the full information setting, where after step $t$ the agent observes all the feedback losses $\ell_\tau^{(t)}$ for all $\tau\leq t$. Our proposed algorithm is a modified version of Exponential Weights \citep{cesa1997use, cesa2006prediction}, summarized in Algorithm \ref{alg:ewa}. The idea is for the agent to continuously update their beliefs on the true loss, even retroactively, as more feedback is presented.

Specifically, at each step $t$, the agent maintains a probability distribution $p$ over the set of possible actions as a function of an estimated total loss $L\in \mathbb{R}_{+}^K$:
\begin{align}\label{eq:prob-def}
p_i(L) = \frac{e^{-\eta L_i}}{\sum_{j\in[K]}e^{-\eta L_j}}.
\end{align}
where $\eta$ is some learning rate chosen by the agent. In our case, the agent's estimation of the total loss is based on the most recently observed feedback losses:
\[ L^\mathrm{e}_t = \sum_{\tau=1}^{t-1}{\ell_\tau^{(t-1)}}. \]

\begin{algorithm}
\caption{\label{alg:ewa} Evolving Exponential Weights}
\begin{algorithmic}[1]
\Parameters{$K, T\in \mathbb{N}$; $\eta>0$.}
\Init{$L^\mathrm{e}_{1,i} \gets 0$ for all $i \in [K]$.}
\For{$t = 1$ \textbf{to} $T$}
    \State Compute $p_i\left(L^\mathrm{e}_{t}\right) = \frac{e^{-\eta L^\mathrm{e}_{t,i}}}{\sum_{j\in[K]}e^{-\eta L^\mathrm{e}_{t,j}}}$ for all $i\in [K]$.
    \State Play a random action $a_t \sim p\left(L^\mathrm{e}_{t}\right)$ and observe $\ell_\tau^{(t)}$ for all $\tau\leq t$.
    \State Compute $L^\mathrm{e}_{t+1} = \sum_{\tau=1}^{t}{\ell_\tau^{(t)}}$.
\EndFor
\end{algorithmic}
\end{algorithm}

To quantify the regret, we will use the \emph{total feedback inaccuracy} $D$ of the adversary:
\[
D \triangleq \sum_{t=1}^T{\left\|L^\mathrm{e}_t - L_t\right\|_\infty},
\]
denoting the true total loss up to step $t$ by $L_t = \sum_{\tau=1}^{t-1}{\ell_\tau}$. Note this term does not depend on the agent's actions, but only on the losses generated by the adversary. It captures the magnitude of the difference between the observed and true losses.

In the case of a delayed setting with delay $d_t$ we have that $\left\|L^\mathrm{e}_t - L_t\right\|_\infty \leq d_t$, and thus $D~\leq~\sum_{t=1}^T d_t$, generalizing the total delay term.

\subsection{Analysis}
To analyze the regret of Algorithm \ref{alg:ewa}, we will start by separating the regret into an \emph{observation drift term} and an \emph{auxiliary regret} term, as usually done when analyzing delayed settings. We can present the expected regret, assuming $a^*$ is the optimal action, as:
\begin{align}\label{eq:full-separation}
R(T) &= \mathbb{E}\left[\sum_{t=1}^T{(\ell_{t,a_t} - \ell_{t,{a^*}})}\right] \notag \\
&= \sum_{t=1}^T{\left(p(L^\mathrm{e}_{t})\cdot\ell_t - \ell_{t,{a^*}}\right)} \notag \\
&= \underbrace{\sum_{t=1}^T{\left(p(L^\mathrm{e}_{t}) - p(L_{t})\right)\cdot\ell_t}}_{\text{observation drift}} + \underbrace{\sum_{t=1}^T{\left(p(L_{t})\cdot\ell_t - \ell_{t,{a^*}}\right)}}_{\text{auxiliary regret}}.
\end{align}

A standard Exponential Weights analysis bounds the auxiliary regret:
\begin{restatable}{lemma}{lemstandardexp}\label{lem:standard-exp3}
Computing $p$ as in Eq. (\ref{eq:prob-def}), we have for any action $a\in[K]$:
\[ \sum_{t=1}^T{\left(p(L_{t})\cdot\ell_t - \ell_{t,a}\right)} \leq \frac{\ln{K}}{\eta} + \frac{\eta}{2}T. \]
\end{restatable}

To bound the drift term, we will use the following lemma:
\begin{restatable}{lemma}{lemdriftbound}\label{lem:drift-bound}
Let $\ell\in[0,1]^K$ be some loss vector, and $L_1,L_2$ two different estimations of the total loss. If the probability $p$ is computed as described in Eq. (\ref{eq:prob-def}), we get:
\[ \left(p(L_1)-p(L_2)\right)\cdot\ell \leq 2\eta\left\|L_1-L_2\right\|_\infty. \]
\end{restatable}

By substituting the results of Lemmas \ref{lem:standard-exp3} and \ref{lem:drift-bound} in Eq. (\ref{eq:full-separation}), we can now derive the main result for this section:
\begin{theorem}\label{thm:full-info-bound}
    The expected regret of Algorithm \ref{alg:ewa} after $T$ rounds holds:
    \[ R(T) \leq \frac{\ln{K}}{\eta} + \eta \sum_{t=1}^T{\left(\frac{1}{2} + 2\left\|L^\mathrm{e}_t-L_t\right\|_\infty\right)} = \frac{\ln{K}}{\eta} + \eta \left(\frac{T}{2} + 2D\right).  \]
\end{theorem}

Optimally, we would like to set $\eta = \sqrt{\frac{\ln{K}}{\frac{T}{2} + 2D}}$, but we do not necessarily know the value of $D$ ahead of time. We can either use a standard doubling trick (See \citep{bistritz2019online, lancewicki2022learning}) to get a parameter-independent algorithm, or use a known upper bound:
\begin{corollary}\label{cor:ewa-regret}
    Let $\bar{D}$ be a known upper bound on the feedback inaccuracy, such that:
    \[ \bar{D} \geq D = \sum_{t=1}^T{\left\|L^\mathrm{e}_t-L_t\right\|_\infty}. \]
    Using $\eta = \sqrt{\frac{\ln{K}}{\frac{T}{2} + 2\bar{D}}}$, the expected regret of Algorithm \ref{alg:ewa} after $T$ rounds holds:
    \[
        R(T) \leq \sqrt{4\ln{K}\left(\frac{T}{2} + 2\bar{D}\right)}.
    \]
\end{corollary}

In the delayed setting where $D$ is the total delay, this bound is known to be optimal \citep{cesa2016delay}. Note that in the worst case, $D=\Omega\left(T^2\right)$. This implies that if the adversary provides no feedback until time $T$ on any of the losses, the regret is potentially $\Theta(T)$, as should be expected in such a scenario.

\section{Evolving FTRL}\label{sec:ftrl}
In the bandit setting, the agent only receives feedback losses for actions they played at the times they were played. So if at time $\tau$ that agent chose action $a_\tau$, it will now receive $\ell^{(t)}_{\tau,a_\tau}$ for all $t\geq \tau$, and will \emph{not} receive $\ell^{(t)}_{\tau,a}$ for any other action $a\neq a_\tau$.

We will use a Follow-The-Regularized-Leader (FTRL) \citep{abernethy2008competing, orabona2019modern} variant as a strategy, presented in Algorithm \ref{alg:ftrl}. Given an estimated total loss $L\in\mathbb{R}_{+}^K$, we compute the probabilities over the set of actions as follows:
\begin{align}\label{eq:ftrl-prob-def}
p(L) = \argmin_{p\in\Delta_K}{\left(p\cdot L + \Phi\left(p\right)\right)},
\end{align}
where $\Delta_K$ is the $K$-dimensional simplex and $\Phi$ is the \emph{regularization function}. Specifically in our analysis, we will use a standard negative entropy with log barrier regularization:
\begin{align}\label{eq:def-reg}
\Phi_{\eta,\gamma}\left(p\right) \triangleq \sum_{i\in[K]}{\left(\frac{p_i}{\eta} - \frac{1}{\gamma}\right)\ln{p_i}}
\end{align}
for some parameters $\eta,\gamma>0$.

\paragraph{Loss estimates}
Similar to what we did in the full-information case, as an estimation for the total loss at step $t$ we will use the most recent update:
\[ \widehat{L}^\mathrm{e}_t = \sum_{\tau=1}^{t-1} \hat{\ell}^{(t-1)}_\tau, \]
where $\hat{\ell}^{(t-1)}_\tau$ is an unbiased loss estimate calculated as:
\[ \hat{\ell}^{(t-1)}_{\tau,a} = \ell^{(t-1)}_{\tau,a} \frac{\mathds{1}\left[a=a_\tau\right]}{p_{a_\tau}\left(\widehat{L}^\mathrm{e}_\tau\right)} \]
for all $\tau< t$ and $a\in[K]$. Although these definitions might appear circular at first, we highlight that to compute $\widehat{L}^\mathrm{e}_t$ there is a need only for prior values of $\widehat{L}^\mathrm{e}_\tau$ where $\tau<t$.

Useful for our analysis is also $\hat{\ell}_\tau$, the unbiased loss estimate when the feedback is accurate:
\[ \hat{\ell}_{\tau,a} = \ell_{\tau,a} \frac{\mathds{1}\left[a=a_\tau\right]}{p_{a_\tau}\left(\widehat{L}^\mathrm{e}_\tau\right)}. \]

\paragraph{Feedback accuracy measure}
In contrast to the full-information setting, we will quantify the feedback accuracy differently, using what we call the \emph{feedback inaccuracy coefficients}:
\[
\lambda_\tau^{(t)} \triangleq \frac{\left\|\ell_\tau^{(t)} - \ell_\tau \right\|_2}{1 + \left\|\ell_\tau^{(t)} - \ell_\tau \right\|_2} = \Theta\left(\min\left\{1,\left\|\ell_\tau^{(t)} - \ell_\tau \right\|_2\right\}\right),
\]
measuring the feedback inaccuracy at step $t$ about the losses of step $\tau$. We again emphasize that the value of $\lambda_\tau^{(t)}$ does not depend on the agent's actions.

We will also denote by $\lambda_t~=~\sum_{\tau=1}^{t-1}{\lambda_\tau^{(t-1)}}$ the total feedback inaccuracy measure at step $t$, and by $d_\mathrm{max}~=~\max\left\{d \mid \exists_t \left(\ell_{t-d}^{(t)} \neq \ell_{t-d}\right) \right\}$ the maximal amount of rounds that the feedback can evolve.

Again taking an example from the delayed setting, we can see that $\lambda_t \leq d_t$ and thus the accuracy measures generalize the delays.

\begin{algorithm}
\caption{\label{alg:ftrl} Evolving FTRL}
\begin{algorithmic}[1]
\Parameters{Function $\Phi$; $K, T\in \mathbb{N}$.}
\Init{$\widehat{L}^\mathrm{e}_{1,i} \gets 0$ for all $i \in [K]$.}
\For{$t = 1$ \textbf{to} $T$}
    \State Compute $p\left(\widehat{L}^\mathrm{e}_{t}\right) = \argmin_{p\in\Delta_K}{\left(p\cdot \widehat{L}^\mathrm{e}_{t} + \Phi\left(p\right)\right)}$.
    \State Play a random action $a_t \sim p\left(\widehat{L}^\mathrm{e}_{t}\right)$ and observe $\ell_{\tau,a}^{(t)}$ for all $\tau\leq t$ such that $a=a_\tau$.
    \State Compute $\hat{\ell}^{(t)}_{\tau,a} = \ell^{(t)}_{\tau,a} \frac{\mathds{1}\left[a=a_\tau\right]}{p_{a_\tau}\left(\widehat{L}^\mathrm{e}_\tau\right)}$ for all $\tau\leq t$ and $a\in[K]$.
    \State Compute $\widehat{L}^\mathrm{e}_{t+1} = \sum_{\tau=1}^t{\hat{\ell}_\tau^{(t)}}$.
\EndFor
\end{algorithmic}
\end{algorithm}

\subsection{Analysis}
For the analysis of Algorithm \ref{alg:ftrl}, we will follow a method similar to \citep{van2023unified}. We will separate the regret into different drift terms, bounding each independently. Our novelty comes from new intermediate loss estimates, parameterized by the feedback accuracy coefficients $\lambda_\tau^{(t)}$. Namely, we will use the following intermediate loss estimates:
\begin{align*}
\widetilde{L}^\mathrm{e}_t &= \sum_{\tau=1}^{t-1}{\left(\left(1-\lambda_\tau^{(t-1)}\right)\hat{\ell}_\tau^{(t-1)} + \lambda_\tau^{(t-1)}\ell_\tau^{(t-1)}\right)}, \\
\widetilde{L}_t &= \sum_{\tau=1}^{t-1}{\left(\left(1-\lambda_\tau^{(t-1)}\right)\hat{\ell}_\tau + \lambda_\tau^{(t-1)}\ell_\tau\right)}, \\
\widehat{L}_t &= \sum_{\tau=1}^{t-1}{\hat{\ell}_\tau}, \qquad \mbox{ and } \qquad
\widehat{L}^*_t = \widehat{L}_t  + \hat{\ell}_t = \sum_{\tau=1}^{t}{\hat{\ell}_\tau}.
\end{align*}

\subsubsection{Drift terms}
We can now represent the regret as (again denoting the optimal action by $a^*$):
\begin{align}\label{eq:ftrl-first-separation}
R(T) &= \mathbb{E}\left[\sum_{t=1}^T{(\ell_{t,a_t} - \ell_{t,{a^*}})}\right] \notag \\
&= \sum_{t=1}^T{\mathbb{E}\left[p\left(\widehat{L}^\mathrm{e}_{t}\right)\cdot\ell_t - \ell_{t,{a^*}}\right]} \notag \\
&= \sum_{t=1}^T{\mathbb{E}\left[\left(p\left(\widehat{L}^\mathrm{e}_{t}\right) -p\left(\widehat{L}_{t}\right)\right)\cdot\ell_t\right]} + \sum_{t=1}^T{\mathbb{E}\left[p\left(\widehat{L}_{t}\right)\cdot\ell_t - \ell_{t,{a^*}}\right]}.
\end{align}

Note that
\[
\mathbb{E}\left[\mathds{1}\left[a=a_\tau\right] \mid a_0,\dots,a_{\tau-1}\right] = \mathrm{Pr}\left[a=a_\tau \mid a_0,\dots,a_{\tau-1}\right] = p_{a_\tau}\left(\widehat{L}^\mathrm{e}_\tau \right),
\]
and thus our loss estimate is indeed unbiased relative to the feedback losses:
\[
\mathbb{E}\left[\hat{\ell}^{(t-1)}_\tau\right] = \mathbb{E}\left[\ell^{(t-1)}_{\tau,a} \frac{\mathds{1}\left[a=a_\tau\right]}{p_{a_\tau}\left(\widehat{L}^\mathrm{e}_\tau\right)}\right] = \mathbb{E}\left[\ell^{(t-1)}_{\tau,a} \frac{p_{a_\tau}\left(\widehat{L}^\mathrm{e}_\tau\right)}{p_{a_\tau}\left(\widehat{L}^\mathrm{e}_\tau\right)}\right] = \ell^{(t-1)}_\tau.
\]
In the same way $\mathbb{E}\left[\hat{\ell}_\tau\right] = \ell_\tau$, and hence, continuing Eq. (\ref{eq:ftrl-first-separation}):
\begin{equation}\label{eq:ftrl-regret-separation}
\begin{aligned}
R(T) &= \sum_{t=1}^T{\mathbb{E}\left[\left(p\left(\widehat{L}^\mathrm{e}_{t}\right) -p\left(\widehat{L}_{t}\right)\right)\cdot\ell_t\right]} + \sum_{t=1}^T{\mathbb{E}\left[p\left(\widehat{L}_{t}\right)\cdot\hat{\ell}_t - \hat{\ell}_{t,{a^*}}\right]} \\
&= \underbrace{\sum_{t=1}^T{\mathbb{E}\left[\left(p\left(\widehat{L}^\mathrm{e}_{t}\right) -p\left(\widetilde{L}^\mathrm{e}_{t}\right)\right)\cdot\ell_t\right]}}_{H_1} + \underbrace{\sum_{t=1}^T{\mathbb{E}\left[\left(p\left(\widetilde{L}^\mathrm{e}_{t}\right) - p\left(\widetilde{L}_{t}\right)\right)\cdot\ell_t\right]}}_{H_2} \\
&+ \underbrace{\sum_{t=1}^T{\mathbb{E}\left[\left(p\left(\widetilde{L}_{t}\right) - p\left(\widehat{L}_{t}\right)\right)\cdot\ell_t\right]}}_{H_3} + \underbrace{\sum_{t=1}^T{\mathbb{E}\left[\left(p\left(\widehat{L}_{t}\right) - p\left(\widehat{L}^*_{t}\right)\right)\cdot\hat{\ell}_t\right]}}_{H_4} \\
&+ \underbrace{\mathbb{E}\left[\sum_{t=1}^T{p\left(\widehat{L}^*_{t}\right)\cdot\hat{\ell}_t - \hat{\ell}_{t,{a^*}}}\right]}_{\text{cheating regret}}.
\end{aligned}
\end{equation}

\subsubsection{Bounds}
As we can see from the above equation, we have a cheating regret and four drift terms $H_1,H_2,H_3,H_4$. We bound the cheating regret in the following lemma:
\begin{restatable}{lemma}{lemcheatingbound}\label{lem:cheating-bound}
Computing $p$ as in Eq. (\ref{eq:ftrl-prob-def}) and using regularization $\Phi_{\eta,\gamma}$ as in Eq. (\ref{eq:def-reg}) for some $\eta,\gamma>0$, we get:
\[
\mathbb{E}\left[\sum_{t=1}^T{p\left(\widehat{L}^*_{t}\right)\cdot\hat{\ell}_t - \hat{\ell}_{t,{a^*}}}\right] \leq 1 + \frac{K\ln{T}}{\gamma} + \frac{\ln{K}}{\eta}.
\]
\end{restatable}

For the drift terms, we use the following:
\begin{restatable}{lemma}{lemftrldrift}\label{lem:ftrl-drift-bounds}
Computing $p$ as in Eq. (\ref{eq:ftrl-prob-def}) and using regularization $\Phi_{\eta,\gamma}$ as in Eq. (\ref{eq:def-reg}) for some $\eta,\gamma>0$ such that $\frac{1}{\sqrt{\gamma}}~\geq~128\left(1+d_\mathrm{max}\right)$, we have for the drift terms in Eq. (\ref{eq:ftrl-regret-separation}):
\begin{align*}
H_1, H_3 &\leq 8 \eta \left(KT + \sum_{t=1}^T{\lambda_t}\right), \\
H_2 &\leq 24 \eta \sum_{t=1}^T{\lambda_t}, \\
H_4 &\leq 8 \eta KT.
\end{align*}
\end{restatable}

Substituting the results of Lemmas \ref{lem:cheating-bound} and \ref{lem:ftrl-drift-bounds} in Eq. (\ref{eq:ftrl-regret-separation}), we get our main result:
\begin{theorem}
    By using regularization function $\Phi_{\eta,\gamma}$ as in Eq. (\ref{eq:def-reg}) for some $\eta,\gamma>0$ such that $\frac{1}{\sqrt{\gamma}}~\geq~128\left(1+d_\mathrm{max}\right)$, the expected regret of Algorithm \ref{alg:ftrl} after $T$ rounds holds:
    \[
    R(T) \leq 1 + \frac{K\ln{T}}{\gamma} + \frac{\ln{K}}{\eta} + \eta\left(24 KT + 40 \sum_{t=1}^T{\lambda_t}\right).
    \]
\end{theorem}

Same as in the full-information case, we can either use a doubling trick or use a known bound on the feedback accuracy as follows.

\begin{corollary}\label{cor:ftrl-regret}
    Let $\bar{\Lambda}~\geq~\sum_{t=1}^T{\lambda_t}$ be a known upper bound on the total inaccuracy. Choosing $\eta~=~\frac{1}{\sqrt{KT + \bar{\Lambda}}}$ and $\gamma=\eta K$, the expected regret of Algorithm \ref{alg:ftrl} using regularization $\Phi_{\eta,\gamma}$ as in Eq. (\ref{eq:def-reg}) holds for any $T~\geq~256K d^4_\mathrm{max}$:
    \[
    R(T) = \widetilde{O}\left(\sqrt{KT + \bar{\Lambda}}\right).
    \]
\end{corollary}

Since in the delayed setting the total delay $\sum_{t=1}^T{d_t} \geq \sum_{t=1}^T{\lambda_t}$, we again capture the optimal asymptotic bound (up to logarithmic terms).

\subsection{Skipping technique}
Note that in cases where the maximal delay $d_\mathrm{max}$ is very large or unbounded, we cannot use Corollary \ref{cor:ftrl-regret} directly.

To accommodate this issue, we will use a skipping wrapper similar to the one used in delayed settings with unbounded delays \citep{thune2019nonstochastic,zimmert2020optimal}, presented in Algorithm \ref{alg:skipping}. The idea is to wrap our regret minimization algorithm to receive new observations only up to a certain delay.

\begin{algorithm}
\caption{\label{alg:skipping} Skipping wrapper}
\begin{algorithmic}[1]
\Parameters{Algorithm $\mathcal{A}$; $T\in \mathbb{N}$; $d_\mathrm{max}>0$.}
\For{$t = 1$ \textbf{to} $T$}
    \State Receive an action distribution $p^\mathcal{A}_t$ from algorithm $\mathcal{A}$.
    \State Play a random action $a_t \sim p^\mathcal{A}_t$.
    \State Feed back to $\mathcal{A}$ all observed $\ell_{\tau,i}^{(t)}$ such that $t \leq \tau + d_\mathrm{max}$. Otherwise feed back $\ell_{\tau,i}^{(\tau + d_\mathrm{max})}$.
\EndFor
\end{algorithmic}
\end{algorithm}

\begin{restatable}{lemma}{lemskipping}\label{lem:skipping}
Denote the regret of some algorithm $\mathcal{A}$ compared to a loss sequence $\ell_1,\dots,\ell_T$ with maximal delay $d_\mathrm{max}$ as $R^\mathcal{A}_{d_\mathrm{max}}\left(\left\{\ell_t\right\}_{1\leq t \leq T}\right)$. When using Algorithm \ref{alg:skipping} with $\mathcal{A}$, the expected regret holds:
\[
R(T) \leq R^\mathcal{A}_{d_\mathrm{max}}\left(\left\{\ell^{(t+d_\mathrm{max})}_t\right\}_{1\leq t \leq T}\right) + 2\sum_{t=1}^T{\left\|\ell_t - \ell^{(t+d_\mathrm{max})}_t\right\|_\infty}.
\]
\end{restatable}

We thus obtain a small regret bound when the estimation accuracy improves with time and becomes very accurate for large delays.
\begin{corollary}
     Denote $d_\mathrm{max}=\lfloor\frac{1}{4}\left(\frac{T}{K}\right)^\frac{1}{4}\rfloor$ and let $\bar{\Lambda}~\geq~\sum_{t=1}^T{\lambda_t}$ be a known upper bound on the total inaccuracy (ignoring inaccuracies for a delay larger than $d_\mathrm{max}$). Assume $\left\|\ell_\tau - \ell^{(t)}_\tau\right\|_\infty \leq \varepsilon$ for any $t\geq \tau + d_\mathrm{max}$. Then using Algorithm \ref{alg:ftrl} wrapped in Algorithm \ref{alg:skipping} with $\eta~=~\frac{1}{\sqrt{KT + \bar{\Lambda}}}$ and $\gamma=\eta K$, the expected regret holds for any $T$:
    \[
    R(T) = \widetilde{O}\left(\sqrt{KT+\bar{\Lambda}} + \varepsilon T\right).
    \]
\end{corollary}

\section{Applications}\label{sec:bounds}
Our framework generalizes many established online learning environments, some of which we present here.

\subsection{Optimistic delayed feedback}
Previously investigated in \citep{flaspohler2021online,hsieh2022multi} for the full-information setting is the optimistic delayed framework. In this model, the feedback is delayed by $d$ steps, but the agent has access to hints about it after choosing the action.

Specifically, at time $t$ the agent receives a hint $\tilde{\ell}_t \in [0,1]^K$ that estimates $\ell_t$, before observing $\ell_t$ at time $t+d$. Applying it to our framework, we can define $\ell_\tau^{(t)} = \tilde{\ell}_\tau$ for any $\tau \leq t < \tau + d$, and $\ell_\tau^{(t)} = \ell_\tau$ otherwise.

Using Corollaries \ref{cor:ewa-regret} and \ref{cor:ftrl-regret}, we then obtain a full-information regret bound similar to \citep{flaspohler2021online}, and a \emph{newly established} regret bound for the bandit setting. 
\begin{corollary}
Using Algorithm \ref{alg:ewa} with optimistic delayed feedback in the full-information setting guarantees an expected regret of:
\[
\widetilde{O}\left(\sqrt{\sum_{t=1}^{T} {\left\|\sum_{\tau=t-d+1}^{t}{\left(\ell_\tau - \tilde{\ell}_\tau\right)}\right\|_\infty}}\right).
\]

In the bandit setting, using Algorithm \ref{alg:ftrl} guarantees an expected regret of:
\[
\widetilde{O}\left(\sqrt{\sum_{t=1}^{T} {\sum_{\tau=t-d+1}^{t}{\min\left\{1, \left\|\ell_\tau - \tilde{\ell}_\tau\right\|_2\right\}}}}\right).
\]
\end{corollary}

\subsection{Corrupted feedback}
In the corrupted feedback setting \citep{resler2019adversarial,hajiesmaili2020adversarial}, true losses are never revealed, only some corrupted loss $\tilde{\ell}_t \in [0,1]^K$ that is observed immediately. In terms of the evolving feedback framework, this is equivalent to having $\ell_\tau^{(t)} = \tilde{\ell}_\tau$ for any $\tau \leq t$.

To measure the amount of corruption, we denote the \emph{corruption budget} by
\[
\mathcal{C} \triangleq \sum_{t=1}^T \left\|\ell_t - \tilde{\ell}_t\right\|_\infty.
\]

Since the true loss is never revealed and the maximal delay is infinite, we need only the result of the skipping technique (Lemma \ref{lem:skipping}) to obtain a regret bound.
\begin{corollary}
Using Algorithm \ref{alg:skipping} with $d_\mathrm{max}=0$ and any multi-armed bandit $\widetilde{O}\left(\sqrt{KT}\right)$ regret minimization algorithm, we obtain an
\[
\widetilde{O}\left(\sqrt{KT} + \mathcal{C}\right)
\]
expected regret in a corrupted environment.
\end{corollary}

\subsection{Composite delayed feedback}
The composite delayed feedback setting is the case where each loss is spread into $d$ positive partial losses $\tilde{\ell}^{(1)}_t,\dots,\tilde{\ell}^{(d)}_t$ that sum to $\ell_t$, observed consecutively by the agent. Applying it to our framework, we have that
\[
\ell^{(t)}_\tau = \sum_{s=\tau+1}^{\min\{t+1, \tau+d\}}\tilde{\ell}^{(s-\tau)}_\tau.
\]

Previous works \citep{cesa2018nonstochastic,wang2021adaptive} discuss the case where the observations are \emph{anonymous}. Namely, the agent observes only the sum of partial losses revealed in the current step. This fact does not generalize directly into our evolving feedback framework.

Hence, we will look at the non-anonymous scenario, where each observation can be attributed to a time and action. However, we can remove the limitation that the partial losses must be positive and can accommodate in our framework \emph{negative} partial losses. The only restriction is that $\sum_{s=1}^{\bar{s}}\tilde{\ell}^{(s)}_t~\in~[0,1]^K$ for any $1\leq \bar{s} \leq d$.

We can thus obtain regret bounds using Corollaries \ref{cor:ewa-regret} and \ref{cor:ftrl-regret}.
\begin{corollary}
In a composite feedback environment, allowing negative partial losses, using Algorithm \ref{alg:ewa} in the full-information setting guarantees an expected regret of:
\[
\widetilde{O}\left(\sqrt{(1+d)T}\right).
\]

In the bandit setting, using Algorithm \ref{alg:ftrl} guarantees an expected regret of:
\[
\widetilde{O}\left(\sqrt{(K+d)T}\right).
\]
\end{corollary}


\section{Discussion}\label{sec:discussion}
This work introduces a framework for online learning under adversarial feedback that evolves over time. Our setting generalizes and unifies previously studied models like delayed, corrupted, and composite feedback.

We proposed regret minimization algorithms for both the full information (Algorithm \ref{alg:ewa}) and the bandit (Algorithm \ref{alg:ftrl}) settings, achieving asymptotically optimal regret bounds (up to logarithmic terms) that depend on the average accuracy of the observed feedback compared to the true losses, using a novel analysis.

In addition to providing a unified model for many problems, our framework is beneficial for real-world scenarios such as finance and online advertising. By incorporating all available information on the value of actions, our approach achieves regret bounds that were previously not feasible.

Our work introduces a few follow-up research questions. Mainly, are our regret bounds optimal in terms of the instance-dependent average feedback accuracy? Currently, we can show optimality only in cases where the difference between the agent's estimations and the true loss is large, like in the delayed setting. It is an open question if our bounds could be improved for cases where the difference is small.

Another natural question to ask is whether we can expand our model to accommodate loss estimations for future rounds as well as past ones, and how the optimal regret bounds will behave in this scenario.


\section*{Acknowledgments}
This project has received funding from the European Research Council (ERC) under the European Union’s Horizon 2020 research and innovation program (grant agreement No. 882396), by the Israel Science Foundation, the Yandex Initiative for Machine Learning at Tel Aviv University and a grant from the Tel Aviv University Center for AI and Data Science (TAD).

\bibliographystyle{plain}
\bibliography{refs}


\appendix
\section{Deferred proofs from Section \ref{sec:ewa}}\label{sec:ewa-proofs}
\lemstandardexp*
\begin{proof}
For any $t$, we have:
\begin{align*}
\frac{\sum_{j\in[K]}e^{-\eta L_{t+1,j}}}{\sum_{j\in[K]}e^{-\eta L_{t,j}}} &= \frac{\sum_{j\in[K]}e^{-\eta L_{t,j}}e^{-\eta \ell_{t,j}}}{\sum_{j\in[K]}e^{-\eta L_{t,j}}} \\
&= \sum_{j\in[K]}p(L_{t,j}) e^{-\eta \ell_{t,j}} \\
&\leq \sum_{j\in[K]}{p(L_{t,j}) (1-\eta \ell_{t,j}+\frac{\eta^2}{2}\ell_{t,j}^2)} \\
&\leq 1 - \eta \left(p(L_{t})\cdot\ell_t\right) + \frac{\eta^2}{2},
\end{align*}
where we used the fact that $e^{-x} \leq 1-x+\frac{x^2}{2}$ for $x\geq 0$.

For any $a\in[K]$, we thus have:
\begin{align*}
\frac{e^{-\eta L_{T,a}}}{K} &\leq \frac{\sum_{j\in[K]}e^{-\eta L_{T,j}}}{K} \\
&= \prod_{t=1}^T{\frac{\sum_{j\in[K]}e^{-\eta L_{t+1,j}}}{\sum_{j\in[K]}e^{-\eta L_{t,j}}}} \\
&\leq \prod_{t=1}^T{\left(1 - \eta \left(p(L_{t})\cdot\ell_t\right) + \frac{\eta^2}{2}\right)}.
\end{align*}
Taking logs of both sides and using the fact that $\ln{(1+x)} \leq x$, we get the desired result.
\end{proof}

\lemdriftbound*
\begin{proof}
We start by noting that:
\begin{align}\label{eq:drift-bound-proof-1}
\left(p(L_1)-p(L_2)\right)\cdot\ell = \sum_{i\in[K]}{\left(p_i(L_1)-p_i(L_2)\right)\ell_i} = \sum_{i\in[K]}{p_i(L_1)\left(1-\frac{p_i(L_2)}{p_i(L_1)}\right)\ell_i}.
\end{align}
To bound this term, we will use the fact that:
\begin{align*}
\frac{p_i(L_2)}{p_i(L_1)} &= \frac{e^{-\eta {L_2}_i}}{e^{-\eta {L_1}_i}} \frac{\sum_{j\in[K]}{e^{-\eta {L_1}_j}}}{\sum_{j\in[K]}{e^{-\eta {L_2}_j}}} \\
&= e^{\eta ({L_1}_i- {L_2}_i)} \frac{\sum_{j\in[K]}{e^{-\eta {L_2}_j - \eta ({L_1}_j - {L_2}_j)}}}{\sum_{j\in[K]}{e^{-\eta {L_2}_j}}} \\
&\geq e^{\eta ({L_1}_i- {L_2}_i) - \eta \max_j{({L_1}_j - {L_2}_j)}} \\
&\geq e^{\eta \left(\min_j{({L_1}_j- {L_2}_j)} - \max_j{({L_1}_j - {L_2}_j)}\right)} \\
&\geq e^{-2\eta\left\|L_1-L_2\right\|_\infty} \\
&\geq 1-2\eta\left\|L_1-L_2\right\|_\infty.
\end{align*}
Substituting in Eq. (\ref{eq:drift-bound-proof-1}), we get:
\begin{align*}
\left(p(L_1)-p(L_2)\right)\cdot\ell \\
&\leq \sum_{i\in[K]}{2p_i(L_1)\eta\left\|L_1-L_2\right\|_\infty\ell_i} \\
&= 2\eta\left\|L_1-L_2\right\|_\infty\sum_{i\in[K]}{p_i(L_1)\ell_i} \\
&\leq 2\eta\left\|L_1-L_2\right\|_\infty,
\end{align*}
as required.
\end{proof}

\section{Deferred proofs from Section \ref{sec:ftrl}}\label{sec:ftrl-proofs}
\subsection{Cheating regret bound}

\begin{lemma}\label{lem:be-the-leader}
Computing $p$ as in Eq. (\ref{eq:ftrl-prob-def}), for any fixed probability $q\in \Delta_K$ and regularization function $\Phi$:
\[ \sum_{t=1}^T{p\left(\widehat{L}^*_{t}\right)\cdot\hat{\ell}_t} + \Phi\left(p\left(\widehat{L}^*_{1}\right)\right) \leq \sum_{t=1}^T{q\cdot\hat{\ell}_t} + \Phi\left(q\right). \]
\end{lemma}
\begin{proof}
We use a standard be-the-leader analysis and prove using induction on $T$. The base case $T=1$ follows directly from Eq. (\ref{eq:ftrl-prob-def}), so we assume for any $q\in\Delta_K$:
\[ \sum_{t=1}^{T-1}{p\left(\widehat{L}^*_{t}\right)\cdot\hat{\ell}_t} + \Phi\left(p\left(\widehat{L}^*_{1}\right)\right) \leq \sum_{t=1}^{T-1}{q\cdot\hat{\ell}_t} + \Phi\left(q\right). \]

Specifically, we can assign $q=p\left(\widehat{L}^*_{T}\right)$. Hence:
\begin{align*}
\sum_{t=1}^T{p\left(\widehat{L}^*_{t}\right)\cdot\hat{\ell}_t} + \Phi\left(p\left(\widehat{L}^*_{1}\right)\right) &= \sum_{t=1}^{T-1}{p\left(\widehat{L}^*_{t}\right)\cdot\hat{\ell}_t} + \Phi\left(p\left(\widehat{L}^*_{1}\right)\right) + p\left(\widehat{L}^*_{T}\right)\cdot\hat{\ell}_t \\
&\leq \sum_{t=1}^{T-1}{p\left(\widehat{L}^*_{T}\right)\cdot\hat{\ell}_t} + \Phi\left(p\left(\widehat{L}^*_{T}\right)\right) + p\left(\widehat{L}^*_{T}\right)\cdot\hat{\ell}_t \\
&= p\left(\widehat{L}^*_{T}\right)\cdot\widehat{L}^*_{T} + \Phi\left(p\left(\widehat{L}^*_{T}\right)\right) \\
&= \min_{q\in\Delta_K}{\left(q\cdot\widehat{L}^*_{T} + \Phi\left(q\right)\right)}
\end{align*}
as required, where the last equality is from Eq. (\ref{eq:ftrl-prob-def}).
\end{proof}

\lemcheatingbound*
\begin{proof}
Denote by $q^*$ the action probability that chooses $a^*$ with probability $1$, and let
\[ q=\left(1-\frac{1}{T}\right)q^* + \frac{1}{T}p\left(\widehat{L}^*_{1}\right).\]
We get:
\begin{align*}
\sum_{t=1}^T{\left(p\left(\widehat{L}^*_{t}\right)\cdot\hat{\ell}_t - \hat{\ell}_{t,{a^*}}\right)} &= \sum_{t=1}^T{\left(p\left(\widehat{L}^*_{t}\right) - q^*\right)\cdot\hat{\ell}_t} \\
&= \sum_{t=1}^T{\left(p\left(\widehat{L}^*_{t}\right) - q\right)\cdot\hat{\ell}_t} + \sum_{t=1}^T{\left(q-q^*\right)\cdot\hat{\ell}_t} \\
&= \sum_{t=1}^T{\left(p\left(\widehat{L}^*_{t}\right) - q\right)\cdot\hat{\ell}_t} + \frac{1}{T}\sum_{t=1}^T{\left(p\left(\widehat{L}^*_{1}\right)-q^*\right)\cdot\hat{\ell}_t} \\
&\leq \sum_{t=1}^T{\left(p\left(\widehat{L}^*_{t}\right) - q\right)\cdot\hat{\ell}_t} + 1 \\
&\leq \Phi_{\eta,\gamma}\left(q\right) - \Phi_{\eta,\gamma}\left(p\left(\widehat{L}^*_{1}\right)\right) + 1,
\end{align*}
where we used Lemma \ref{lem:be-the-leader} in the last step. Thus, using Eq. (\ref{eq:def-reg}):
\begin{align*}
\sum_{t=1}^T{\left(p\left(\widehat{L}^*_{t}\right)\cdot\hat{\ell}_t - \hat{\ell}_{t,{a^*}}\right)} &\leq 1+ \sum_{i\in[K]}{\left(\frac{q_i}{\eta} - \frac{1}{\gamma}\right)\ln{q_i}} - \sum_{i\in[K]}{\left(\frac{p_i\left(\widehat{L}^*_{1}\right)}{\eta} - \frac{1}{\gamma}\right)\ln{p_i\left(\widehat{L}^*_{1}\right)}} \\
&\leq 1 + \frac{1}{\gamma}\sum_{i\in[K]}\ln{\frac{p_i\left(\widehat{L}^*_{1}\right)}{q_i}} + \sum_{i\in[K]}{\frac{p_i\left(\widehat{L}^*_{1}\right)}{\eta} \ln{\frac{1}{p_i\left(\widehat{L}^*_{1}\right)}}} \\ 
&= 1 + \frac{1}{\gamma}\sum_{i\in[K]}\ln{\frac{p_i\left(\widehat{L}^*_{1}\right)}{\left(1-\frac{1}{T}\right)q^*_i + \frac{1}{T}p_i\left(\widehat{L}^*_{1}\right)}} + \sum_{i\in[K]}{\frac{p_i\left(\widehat{L}^*_{1}\right)}{\eta} \ln{\frac{1}{p_i\left(\widehat{L}^*_{1}\right)}}} \\
&\leq 1 + \frac{K\ln{T}}{\gamma} + \frac{1}{\eta}\sum_{i\in[K]}{p_i\left(\widehat{L}^*_{1}\right) \ln{\frac{1}{p_i\left(\widehat{L}^*_{1}\right)}}}.
\end{align*}
Using Jensen's inequality, we can complete the proof:
\begin{align*}
\sum_{t=1}^T{\left(p\left(\widehat{L}^*_{t}\right)\cdot\hat{\ell}_t - \hat{\ell}_{t,{a^*}}\right)} &\leq 1 + \frac{K\ln{T}}{\gamma} + \frac{1}{\eta}\ln{\sum_{i\in[K]}{{\frac{p_i\left(\widehat{L}^*_{1}\right)}{p_i\left(\widehat{L}^*_{1}\right)}}}} \\
&= 1 + \frac{K\ln{T}}{\gamma} + \frac{\ln{K}}{\eta}.
\end{align*}
\end{proof}

\subsection{Drift bounds preliminaries}
We start with defining the dual norms on $x\in \mathbb{R}^K$ induced by a strictly-convex twice-differentiable regularization function $\Phi$ and a point $p\in \mathbb{R}^K$:
\begin{align*}
\|x\|_{\Phi,p} \triangleq \sqrt{x^T \left(\nabla^2 \Phi\left(p\right)\right)^{-1} x} \qquad \mbox{and} \qquad \|x\|^*_{\Phi,p} \triangleq \sqrt{x^T \left(\nabla^2 \Phi\left(p\right)\right) x},
\end{align*}
where $\nabla^2 \Phi$ denotes the Hessian matrix of $\Phi$.

For $\Phi_{\eta, \gamma}$ as defined in Eq. (\ref{eq:def-reg}) we get:
\begin{align}\label{eq:def-norm}
\|x\|_{\Phi_{\eta, \gamma},p} = \sqrt{\sum_{i\in [K]} \frac{\eta\gamma p_i^2}{\eta + \gamma p_i} x_i^2}\qquad \mbox{and} \qquad\|x\|^*_{\Phi_{\eta, \gamma},p} = \sqrt{\sum_{i\in [K]} \frac{\eta + \gamma p_i}{\eta\gamma p_i^2} x_i^2}.
\end{align}

For clarity, we will also denote the Dikin ellipsoid of radius $\frac{1}{2}$ as:
\[
\mathcal{D}_\Phi\left(p\right) \triangleq \left\{x\in \mathbb{R}^K \mid \|x - p\|^*_{\Phi,p} \leq \frac{1}{2} \right\}.
\]

We will use the following facts (for proofs see Lemma 16, Lemma 1 and Lemma 9 in \citep{van2023unified} respectively):
\begin{fact}\label{fact:close-norms}
    Let $x,p,q\in\mathbb{R}^K$. Using regularization $\Phi_{\eta,\gamma}$ as in Eq. (\ref{eq:def-reg}) for some $\eta,\gamma>0$, we get that if $q \in \mathcal{D}_{\Phi_{\eta,\gamma}}\left(p\right)$, then:
    \[ \frac{1}{2}\|x\|_{\Phi_{\eta,\gamma}, q} \leq \|x\|_{\Phi_{\eta,\gamma}, p} \leq 2\|x\|_{\Phi_{\eta,\gamma}, q}. \]
\end{fact}

\begin{fact}\label{fact:diff-probs}
    Let $L,L'\in \mathbb{R}^K_+$ and $q\in\mathbb{R}^K$. Computing $p$ as in Eq. (\ref{eq:ftrl-prob-def}) and using regularization $\Phi_{\eta,\gamma}$ as in Eq. (\ref{eq:def-reg}) for some $\eta,\gamma>0$, we get that if $p\left(L\right), p\left(L'\right) \in \mathcal{D}_{\Phi_{\eta,\gamma}}\left(q\right)$, then:
    \[ \left\|p(L') - p(L)\right\|^*_{\Phi_{\eta, \gamma},q} \leq 8 \left\|L' - L\right\|_{\Phi_{\eta, \gamma},q}. \]
\end{fact}

\begin{fact}\label{fact:close-probs}
    Let $L,L'\in \mathbb{R}^K_+$. Computing $p$ as in Eq. (\ref{eq:ftrl-prob-def}) and using regularization $\Phi_{\eta,\gamma}$ as in Eq. (\ref{eq:def-reg}) for some $\eta,\gamma>0$, we get that if $\left\|L'-L\right\|_{\Phi_{\eta, \gamma},p\left(L\right)} \leq \frac{1}{16}$, then:
    \[
    p\left(L'\right) \in \mathcal{D}_{\Phi_{\eta,\gamma}}\left(p\left(L\right)\right).
    \]
\end{fact}

\subsection{Drift bounds}
\begin{lemma}\label{lem:close-probs}
 Computing $p$ as in Eq. (\ref{eq:ftrl-prob-def}) and using regularization $\Phi_{\eta,\gamma}$ as in Eq. (\ref{eq:def-reg}) for some $\eta,\gamma>0$ such that $\frac{1}{\sqrt{\gamma}} \geq 32d_\mathrm{max}$, we have for all $0 \leq d \leq d_\mathrm{max}$:
 \[
     p\left(\widehat{L}^\mathrm{e}_{t+d}\right) \in \mathcal{D}_{\Phi_{\eta,\gamma}}\left(p\left(\widehat{L}^\mathrm{e}_t\right)\right).
 \]
\end{lemma}
\begin{proof}
We will prove by induction on $d$. The base case $d=0$ is trivially true. We will thus assume the claim is true for any $d' < d$. Using Fact \ref{fact:close-probs}, we only need to show that:
\[
\sum_{\tau=t}^{t+d-1}{\left\|\hat{\ell}^{(t+d-1)}_\tau - \hat{\ell}^{(t-1)}_\tau\right\|_{\Phi_{\eta, \gamma},p\left(\widehat{L}^\mathrm{e}_t\right)}} \leq \frac{1}{16}.
\]
 From our assumption, we can use Fact \ref{fact:close-norms} and get:
\begin{align*}
\sum_{\tau=t}^{t+d-1}{\left\|\hat{\ell}^{(t+d-1)}_\tau - \hat{\ell}^{(t-1)}_\tau\right\|_{\Phi_{\eta, \gamma},p\left(\widehat{L}^\mathrm{e}_t\right)}} \leq 2 \sum_{\tau=t}^{t+d-1}{\left\|\hat{\ell}^{(t+d-1)}_\tau - \hat{\ell}^{(t-1)}_\tau\right\|_{\Phi_{\eta, \gamma},p\left(\widehat{L}^\mathrm{e}_\tau\right)}}.
\end{align*}
Hence, from Eq. (\ref{eq:def-norm}):
\begin{align*}
\left\|\hat{\ell}^{(t+d-1)}_\tau - \hat{\ell}^{(t-1)}_\tau\right\|_{\Phi_{\eta, \gamma},p\left(\widehat{L}^\mathrm{e}_\tau\right)}^2 &= \sum_{i\in[K]} \frac{\eta\gamma p_i^2\left(\widehat{L}^\mathrm{e}_\tau\right)}{\eta + \gamma p_i\left(\widehat{L}^\mathrm{e}_\tau\right)} \frac{\left(\ell^{(t+d-1)}_{\tau,i} - \ell^{(t-1)}_{\tau,i}\right)^2 \mathds{1}\left[a=a_\tau\right]}{{p_i^2\left(\widehat{L}^\mathrm{e}_\tau\right)}} \\
&\leq \frac{\eta \gamma}{\eta + \gamma p_{a_\tau}\left(\widehat{L}^\mathrm{e}_\tau\right)} \\
&\leq \gamma,
\end{align*}
and thus we get that 
\[
\sum_{\tau=t}^{t+d-1}{\left\|\hat{\ell}^{(t+d-1)}_\tau - \hat{\ell}^{(t-1)}_\tau\right\|_{\Phi_{\eta, \gamma},p\left(\widehat{L}^\mathrm{e}_t\right)}} \leq 2d\sqrt{\gamma} \leq 2 d_\mathrm{max}\sqrt{\gamma} \leq \frac{1}{32}
\]
as required.
\end{proof}

\lemftrldrift*
\begin{proof}
\item\paragraph{$H_1,H_3.$}
We will prove the bound for $H_3$, and the proof for $H_1$ is identical. First, note that:
\[ \widetilde{L}_t - \widehat{L}^\mathrm{e}_t = \sum_{\tau=1}^{t-1}{\left(\left(1-\lambda_\tau^{(t-1)}\right)\left(\hat{\ell}_\tau - \hat{\ell}^{(t-1)}_\tau\right) + \lambda_\tau^{(t-1)}\left(\ell_\tau - \hat{\ell}^{(t-1)}_\tau\right)\right)}. \]
Denote $t'=\max\left\{1,t-d_\mathrm{max}\right\}$. For any $\tau < t'$ we have that $\hat{\ell}_\tau = \hat{\ell}^{(t-1)}_\tau$ and $\lambda_\tau^{(t-1)} = 0$, and thus:
\begin{align*}
\left\|\widetilde{L}_t - \widehat{L}^\mathrm{e}_t\right\|_{\Phi_{\eta, \gamma},p\left(\widehat{L}^\mathrm{e}_t\right)} &\leq \sum_{\tau=t'}^{t-1}{\left\|\hat{\ell}_\tau - \hat{\ell}^{(t-1)}_\tau\right\|_{\Phi_{\eta, \gamma},p\left(\widehat{L}^\mathrm{e}_t\right)}} + \sum_{\tau=t'}^{t-1}{\left\|\ell_\tau - \hat{\ell}^{(t-1)}_\tau\right\|_{\Phi_{\eta, \gamma},p\left(\widehat{L}^\mathrm{e}_t\right)}} \\
&\leq \sum_{\tau=t'}^{t-1}{\left\|\hat{\ell}_\tau\right\|_{\Phi_{\eta, \gamma},p\left(\widehat{L}^\mathrm{e}_t\right)}} + \sum_{\tau=t'}^{t-1}{\left\|\ell_\tau\right\|_{\Phi_{\eta, \gamma},p\left(\widehat{L}^\mathrm{e}_t\right)}} + 2\sum_{\tau=t'}^{t-1}{\left\|\hat{\ell}^{(t-1)}_\tau\right\|_{\Phi_{\eta, \gamma},p\left(\widehat{L}^\mathrm{e}_t\right)}}.
\end{align*}
Using Fact \ref{fact:close-norms} and Lemma \ref{lem:close-probs} we can move to the norm induced by $p\left(\widehat{L}^\mathrm{e}_\tau\right)$:
\begin{align*}
&\left\|\widetilde{L}_t - \widehat{L}^\mathrm{e}_t\right\|_{\Phi_{\eta, \gamma},p\left(\widehat{L}^\mathrm{e}_t\right)} \\
&\leq 2\sum_{\tau=t'}^{t-1}{\left\|\hat{\ell}_\tau\right\|_{\Phi_{\eta, \gamma},p\left(\widehat{L}^\mathrm{e}_\tau\right)}} + \sum_{\tau=t'}^{t-1}{\left\|\ell_\tau\right\|_{\Phi_{\eta, \gamma},p\left(\widehat{L}^\mathrm{e}_t\right)}} + 4\sum_{\tau=t'}^{t-1}{\left\|\hat{\ell}^{(t-1)}_\tau\right\|_{\Phi_{\eta, \gamma},p\left(\widehat{L}^\mathrm{e}_\tau\right)}}.
\end{align*}
We can now use Eq. (\ref{eq:def-norm}) to see that
\begin{align*}
\left\|\hat{\ell}_\tau\right\|_{\Phi_{\eta, \gamma},p\left(\widehat{L}^\mathrm{e}_\tau\right)}^2 &= \sum_{i\in[K]} \frac{\eta\gamma p_i^2\left(\widehat{L}^\mathrm{e}_\tau\right)}{\eta + \gamma p_i\left(\widehat{L}^\mathrm{e}_\tau\right)} \frac{\ell_\tau^2 \mathds{1}\left[a=a_\tau\right]}{{p_i^2\left(\widehat{L}^\mathrm{e}_\tau\right)}} \\
&\leq \frac{\eta \gamma}{\eta + \gamma p_{a_\tau}\left(\widehat{L}^\mathrm{e}_\tau\right)} \\
&\leq \gamma,
\end{align*}
and the same is true for $\left\|\hat{\ell}^{(t-1)}_\tau\right\|^2_{\Phi_{\eta, \gamma},p\left(\widehat{L}^\mathrm{e}_\tau\right)}$. Also:
\[
\left\|\ell_\tau\right\|^2_{\Phi_{\eta, \gamma},p\left(\widehat{L}^\mathrm{e}_t\right)} = \sum_{i\in [K]} \frac{\eta\gamma p_i^2\left(\widehat{L}^\mathrm{e}_t\right)}{\eta + \gamma p_i\left(\widehat{L}^\mathrm{e}_t\right)} \ell_{\tau,i}^2 \leq \sum_{i\in [K]} \gamma p_i^2\left(\widehat{L}^\mathrm{e}_t\right) \leq \gamma
\]
as well. Hence:
\[
\left\|\widetilde{L}_t - \widehat{L}^\mathrm{e}_t\right\|_{\Phi_{\eta, \gamma},p\left(\widehat{L}^\mathrm{e}_t\right)} \leq 7d_\mathrm{max}\sqrt{\gamma} \leq \frac{1}{16}.
\]
Similarly:
\begin{align*}
\left\|\widehat{L}_t - \widehat{L}^\mathrm{e}_t\right\|_{\Phi_{\eta, \gamma},p\left(\widehat{L}^\mathrm{e}_t\right)} &\leq \sum_{\tau=1}^{t-1}\left\| \hat{\ell}_\tau - \hat{\ell}^{(t-1)}_\tau \right\|_{\Phi_{\eta, \gamma},p\left(\widehat{L}^\mathrm{e}_t\right)} \\
&= \sum_{\tau=t'}^{t-1}\left\| \hat{\ell}_\tau - \hat{\ell}^{(t-1)}_\tau \right\|_{\Phi_{\eta, \gamma},p\left(\widehat{L}^\mathrm{e}_t\right)} \\
&\leq 2\sum_{\tau=t'}^{t-1}\left\| \hat{\ell}_\tau - \hat{\ell}^{(t-1)}_\tau \right\|_{\Phi_{\eta, \gamma},p\left(\widehat{L}^\mathrm{e}_\tau\right)} \\
&\leq 2d_\mathrm{max}\sqrt{\gamma} \\
&\leq \frac{1}{16}.
\end{align*}

Thus, we can use Hölder's inequality to get:
\begin{align*}
&\left(p\left(\widetilde{L}_{t}\right) - p\left(\widehat{L}_{t}\right)\right)\cdot\ell_t \\
&\leq \left\|p\left(\widetilde{L}_{t}\right) - p\left(\widehat{L}_{t}\right)\right\|^*_{\Phi_{\eta, \gamma},p\left(\widehat{L}^\mathrm{e}_t\right)} \left\|\ell_t\right\|_{\Phi_{\eta, \gamma},p\left(\widehat{L}^\mathrm{e}_t\right)} \\
&\leq 8\left\|\widetilde{L}_{t} - \widehat{L}_{t}\right\|_{\Phi_{\eta, \gamma},p\left(\widehat{L}^\mathrm{e}_t\right)} \left\|\ell_t\right\|_{\Phi_{\eta, \gamma},p\left(\widehat{L}^\mathrm{e}_t\right)} \\
&\leq 8\left\|\sum_{\tau=1}^{t-1}\lambda_\tau^{(t-1)}\left(\ell_\tau - \hat{\ell}_\tau\right)\right\|_{\Phi_{\eta, \gamma},p\left(\widehat{L}^\mathrm{e}_t\right)} \left\|\ell_t\right\|_{\Phi_{\eta, \gamma},p\left(\widehat{L}^\mathrm{e}_t\right)} \\
&= 8\left\|\sum_{\tau=t'}^{t-1}\lambda_\tau^{(t-1)}\left(\ell_\tau - \hat{\ell}_\tau\right)\right\|_{\Phi_{\eta, \gamma},p\left(\widehat{L}^\mathrm{e}_t\right)} \left\|\ell_t\right\|_{\Phi_{\eta, \gamma},p\left(\widehat{L}^\mathrm{e}_t\right)},
\end{align*}
where the second inequality is due to Fact \ref{fact:diff-probs} and last inequality is since $\lambda_\tau^{(t-1)}=0$ for any $\tau<t'$.

Since our estimators are unbiased, $\mathbb{E}\left[\left(\ell_\tau - \hat{\ell}_\tau\right)\left(\ell_{\tau'} - \hat{\ell}_{\tau'}\right)\right] = 0$ for any $\tau\neq\tau'$, and thus
\begin{align*}
\mathbb{E}\left[\left\|\sum_{\tau=t'}^{t-1}\lambda_\tau^{(t-1)}\left(\ell_\tau - \hat{\ell}_\tau\right)\right\|^2_{\Phi_{\eta, \gamma},p\left(\widehat{L}^\mathrm{e}_t\right)}\right] &= \sum_{\tau=t'}^{t-1}\left(\lambda_\tau^{(t-1)}\right)^2\mathbb{E}\left[\left\|\ell_\tau - \hat{\ell}_\tau\right\|^2_{\Phi_{\eta, \gamma},p\left(\widehat{L}^\mathrm{e}_t\right)}\right] \\
&= \sum_{\tau=t'}^{t-1}\left(\lambda_\tau^{(t-1)}\right)^2\mathbb{E}\left[\left\|\hat{\ell}_\tau\right\|^2_{\Phi_{\eta, \gamma},p\left(\widehat{L}^\mathrm{e}_t\right)} - \left\|\ell_\tau\right\|^2_{\Phi_{\eta, \gamma},p\left(\widehat{L}^\mathrm{e}_t\right)}\right] \\
&\leq \sum_{\tau=t'}^{t-1}\left(\lambda_\tau^{(t-1)}\right)^2\mathbb{E}\left[\left\| \hat{\ell}_\tau\right\|^2_{\Phi_{\eta, \gamma},p\left(\widehat{L}^\mathrm{e}_t\right)}\right] \\
&\leq 2\sum_{\tau=t'}^{t-1}\left(\lambda_\tau^{(t-1)}\right)^2\mathbb{E}\left[\left\| \hat{\ell}_\tau\right\|^2_{\Phi_{\eta, \gamma},p\left(\widehat{L}^\mathrm{e}_\tau\right)}\right],
\end{align*}
where in the last step we used Fact \ref{fact:close-probs} and Lemma \ref{lem:close-probs} to move to the norms induced by $p\left(\widehat{L}^\mathrm{e}_\tau\right)$.

Using Eq. (\ref{eq:def-norm}):
\begin{align}\label{eq:exp-hat-bound}
\mathbb{E}\left[\left\|\hat{\ell}_\tau\right\|^2_{\Phi_{\eta, \gamma},p\left(\widehat{L}^\mathrm{e}_\tau\right)}\right] &= \mathbb{E}\left[\sum_{i\in[K]}{ \frac{\eta\gamma p_i^2\left(\widehat{L}^\mathrm{e}_\tau\right)}{\eta + \gamma p_i\left(\widehat{L}^\mathrm{e}_\tau\right)} \frac{\ell_{\tau,i}^2 \mathds{1}\left[a=a_\tau\right]}{p_i^2\left(\widehat{L}^\mathrm{e}_\tau\right)}}\right] \notag \\
&= \mathbb{E}\left[\sum_{i\in[K]}{\frac{\eta\gamma p_i\left(\widehat{L}^\mathrm{e}_\tau\right)}{\eta + \gamma p_i\left(\widehat{L}^\mathrm{e}_\tau\right)} \ell_{\tau,i}^2}\right] \notag \\
&\leq \eta K.
\end{align}
Combining the last equations and using Jensen's inequality, we thus have:
\[
\mathbb{E}\left[\left\|\sum_{\tau=t'}^{t-1}\lambda_\tau^{(t-1)}\left(\ell_\tau - \hat{\ell}_\tau\right)\right\|_{\Phi_{\eta, \gamma},p\left(\widehat{L}^\mathrm{e}_t\right)}\right] \leq \sqrt{2\eta K \sum_{\tau=1}^{t-1}{\left(\lambda_\tau^{(t-1)}\right)^2}} \leq \sqrt{2\eta K \lambda_t} \leq \sqrt{\eta}(K+\lambda_t).
\]
Using Eq. (\ref{eq:def-norm}) again, we have for any $\tau,t$:
\[
\left\|\ell_\tau\right\|^2_{\Phi_{\eta, \gamma},p\left(\widehat{L}^\mathrm{e}_t\right)} = \sum_{i\in [K]} \frac{\eta\gamma p_i^2\left(\widehat{L}^\mathrm{e}_t\right)}{\eta + \gamma p_i\left(\widehat{L}^\mathrm{e}_t\right)} \ell_{\tau,i}^2 \leq \sum_{i\in [K]}\eta p_i\left(\widehat{L}^\mathrm{e}_t\right) = \eta,
\]
so in total, we get 
\[
H_3 = \sum_{t=1}^T{\mathbb{E}\left[\left(p\left(\widetilde{L}_{t}\right) - p\left(\widehat{L}_{t}\right)\right)\cdot\ell_t\right]} \leq 8\eta\left(KT+\sum_{t=1}^T{\lambda_t}\right)
\]
as desired.

\item\paragraph{$H_2.$}
For $H_2$, observe that the same as before, we have:
\begin{align*}
\left\|\widetilde{L}_t - \widehat{L}^\mathrm{e}_t\right\|_{\Phi_{\eta, \gamma},p\left(\widehat{L}^\mathrm{e}_t\right)} \leq \frac{1}{16} \qquad \mbox{and} \qquad \left\|\widetilde{L}^\mathrm{e}_t - \widehat{L}^\mathrm{e}_t\right\|_{\Phi_{\eta, \gamma},p\left(\widehat{L}^\mathrm{e}_t\right)} \leq \frac{1}{16}.
\end{align*}
And thus we can use Hölder's inequality as before:
\begin{align*}
&\left(p\left(\widetilde{L}^\mathrm{e}_{t}\right) - p\left(\widetilde{L}_{t}\right)\right)\cdot\ell_t \\
&\leq 8\left\|\widetilde{L}^\mathrm{e}_{t} - \widetilde{L}_{t}\right\|_{\Phi_{\eta, \gamma},p\left(\widehat{L}^\mathrm{e}_t\right)} \left\|\ell_t\right\|_{\Phi_{\eta, \gamma},p\left(\widehat{L}^\mathrm{e}_t\right)} \\
&\leq 8\sqrt{\eta}\sum_{\tau=1}^{t-1}\left(1-\lambda_\tau^{(t-1)}\right)\left\|\hat{\ell}^{(t-1)}_\tau - \hat{\ell}_\tau\right\|_{\Phi_{\eta, \gamma},p\left(\widehat{L}^\mathrm{e}_t\right)} + 8\sqrt{\eta}\sum_{\tau=1}^{t-1}\lambda_\tau^{(t-1)}\left\|\ell^{(t-1)}_\tau - \ell_\tau\right\|_{\Phi_{\eta, \gamma},p\left(\widehat{L}^\mathrm{e}_t\right)} \\
&\leq 8\sqrt{\eta}\sum_{\tau=1}^{t-1}\left(1-\lambda_\tau^{(t-1)}\right)\left\|\hat{\ell}^{(t-1)}_\tau - \hat{\ell}_\tau\right\|_{\Phi_{\eta, \gamma},p\left(\widehat{L}^\mathrm{e}_t\right)} + 8\eta\lambda_t \\
&= 8\sqrt{\eta}\sum_{\tau=t'}^{t-1}\left(1-\lambda_\tau^{(t-1)}\right)\left\|\hat{\ell}^{(t-1)}_\tau - \hat{\ell}_\tau\right\|_{\Phi_{\eta, \gamma},p\left(\widehat{L}^\mathrm{e}_t\right)} + 8\eta\lambda_t \\
&\leq 16\sqrt{\eta}\sum_{\tau=t'}^{t-1}\left(1-\lambda_\tau^{(t-1)}\right)\left\|\hat{\ell}^{(t-1)}_\tau - \hat{\ell}_\tau\right\|_{\Phi_{\eta, \gamma},p\left(\widehat{L}^\mathrm{e}_\tau\right)} + 8\eta\lambda_t
\end{align*}
where again we denote $t'=\max\left\{1,t-d_\mathrm{max}\right\}$. The second inequality is due to Fact \ref{fact:diff-probs} and $\left\|\ell_t\right\|^2_{\Phi_{\eta, \gamma},p\left(\widehat{L}^\mathrm{e}_t\right)}\leq \eta$, and the last inequality is due to Fact \ref{fact:close-probs} and Lemma \ref{lem:close-probs} to move to the norms induced by $p\left(\widehat{L}^\mathrm{e}_\tau\right)$.

Using Eq. (\ref{eq:def-norm}):
\begin{align*}
\mathbb{E}\left[\left\|\hat{\ell}^{(t-1)}_\tau - \hat{\ell}_\tau\right\|^2_{\Phi_{\eta, \gamma},p\left(\widehat{L}^\mathrm{e}_\tau\right)}\right] &= \mathbb{E}\left[\sum_{i\in[K]}{ \frac{\eta\gamma p_i^2\left(\widehat{L}^\mathrm{e}_\tau\right)}{\eta + \gamma p_i\left(\widehat{L}^\mathrm{e}_\tau\right)} \frac{\left(\ell^{(t-1)}_{\tau,i} - \ell_{\tau,i}\right)^2 \mathds{1}\left[a=a_\tau\right]}{p_i^2\left(\widehat{L}^\mathrm{e}_\tau\right)}}\right] \\
&= \mathbb{E}\left[\sum_{i\in[K]}{\frac{\eta\gamma p_i\left(\widehat{L}^\mathrm{e}_\tau\right)}{\eta + \gamma p_i\left(\widehat{L}^\mathrm{e}_\tau\right)} \left(\ell^{(t-1)}_{\tau,i} - \ell_{\tau,i}\right)^2}\right] \\
&\leq \eta \sum_{i\in[K]}\left(\ell^{(t-1)}_{\tau,i} - \ell_{\tau,i}\right)^2 \\
&= \eta \left\| \ell^{(t-1)}_{\tau} - \ell_{\tau} \right\|^2_2,
\end{align*}
and so by Jensen's inequality:
\[
\mathbb{E}\left[\left\|\hat{\ell}^{(t-1)}_\tau - \hat{\ell}_\tau\right\|_{\Phi_{\eta, \gamma},p\left(\widehat{L}^\mathrm{e}_\tau\right)}\right] \leq 
\sqrt{\eta} \left\| \ell^{(t-1)}_{\tau} - \ell_{\tau} \right\|_2
= \sqrt{\eta} \frac{\lambda^{(t-1)}_{\tau}}{1-\lambda^{(t-1)}_{\tau}}.
\]
Overall we get
\[
H_2 = \sum_{t=1}^{T}\mathbb{E}\left[\left(p\left(\widetilde{L}^\mathrm{e}_{t}\right) - p\left(\widetilde{L}_{t}\right)\right)\cdot\ell_t\right]
\leq 24 \eta \sum_{t=1}^{T} \lambda_t.
\]

\item\paragraph{$H_4.$}
Note that:
\begin{align*}
\left\|\widehat{L}^*_t - \widehat{L}^\mathrm{e}_t\right\|_{\Phi_{\eta, \gamma},p\left(\widehat{L}^\mathrm{e}_t\right)} \leq \left\| \hat{\ell}_t \right\|_{\Phi_{\eta, \gamma},p\left(\widehat{L}^\mathrm{e}_t\right)} + \left\|\widehat{L}_t - \widehat{L}^\mathrm{e}_t\right\|_{\Phi_{\eta, \gamma},p\left(\widehat{L}^\mathrm{e}_t\right)} 
\leq \left(1+2d_\mathrm{max}\right)\sqrt{\gamma} 
\leq \frac{1}{16},
\end{align*}
Ss we can again use Fact \ref{fact:diff-probs}:
\begin{align*}
\mathbb{E}\left[\left(p\left(\widehat{L}_{t}\right) - p\left(\widehat{L}^*_{t}\right)\right)\cdot\hat{\ell}_t\right] &\leq 8\mathbb{E}\left[\left\|\widehat{L}_{t} - \widehat{L}^*_{t}\right\|_{\Phi_{\eta, \gamma},p\left(\widehat{L}^\mathrm{e}_t\right)} \left\|\hat{\ell}_t\right\|_{\Phi_{\eta, \gamma},p\left(\widehat{L}^\mathrm{e}_t\right)}\right] \\
&= 8\mathbb{E}\left[\left\|\hat{\ell}_t\right\|^2_{\Phi_{\eta, \gamma},p\left(\widehat{L}^\mathrm{e}_t\right)}\right] \\
&\leq 8 \eta K
\end{align*}
where the last step is due to Eq. (\ref{eq:exp-hat-bound}). We can now complete the proof with:
\[
H_4 = \sum_{t=1}^T \mathbb{E}\left[\left(p\left(\widehat{L}_{t}\right) - p\left(\widehat{L}^*_{t}\right)\right)\cdot\hat{\ell}_t\right] \leq 8\eta KT.
\]
\end{proof}

\subsection{Skipping bound}
\lemskipping*
\begin{proof}
We have:
\begin{align*}
R(T) &= \max_{a\in[K]}\mathbb{E}\left[{\sum_{t=1}^T{\ell_{t,a_t} - \ell_{t,a}}}\right] \\
&= \max_{a\in[K]}\mathbb{E}\left[\sum_{t=1}^T{\left(\ell^{(t+d_\mathrm{max})}_{t,a_t} - \ell^{(t+d_\mathrm{max})}_{t,a}\right) + \left(\ell_{t,a_t} - \ell^{(t+d_\mathrm{max})}_{t,a_t}\right) + \left(\ell^{(t+d_\mathrm{max})}_{t,a} - \ell_{t,a}\right)}\right] \\
&\leq \max_{a\in[K]}\mathbb{E}\left[\sum_{t=1}^T{\ell^{(t+d_\mathrm{max})}_{t,a_t} - \ell^{(t+d_\mathrm{max})}_{t,a}}\right] + \sum_{t=1}^T{\left\|\ell_t - \ell^{(t+d_\mathrm{max})}_t\right\|_\infty} \\
&= R^\mathcal{A}_{d_\mathrm{max}}\left(\left\{\ell^{(t+d_\mathrm{max})}_t\right\}_{1\leq t \leq T}\right) + 2\sum_{t=1}^T{\left\|\ell_t - \ell^{(t+d_\mathrm{max})}_t\right\|_\infty},
\end{align*}
where the last step is since algorithm $\mathcal{A}$ cannot distinguish between being wrapped in Algorithm \ref{alg:skipping} and $\left\{\ell^{(t+d_\mathrm{max})}_t\right\}_{1\leq t \leq T}$ being the true losses with maximal delay $d_\mathrm{max}$.
\end{proof}

\end{document}